\documentclass{article} % For LaTeX2e
\usepackage[preprint]{colm2026_conference}

\usepackage{microtype}
\usepackage{graphicx}
\usepackage{subcaption}
\usepackage{booktabs}
\usepackage{hyperref}
\usepackage{amsmath}
\usepackage{amssymb}
\usepackage{mathtools}
\usepackage{amsthm}
\usepackage{algorithm}
\usepackage{algorithmic}
\usepackage[capitalize,noabbrev]{cleveref}
\usepackage{multirow}

\theoremstyle{plain}
\newtheorem{theorem}{Theorem}[section]

\newtheorem{lemma}[theorem]{Lemma}

\theoremstyle{definition}

\newcommand{\E}{\mathbb{E}}

\newcommand{\indicator}{\mathbf{1}}
\newcommand{\positives}{\mathcal{S}^{+}}
\newcommand{\negatives}{\mathcal{S}^{-}}
\newcommand{\dataset}{\mathcal{D}}
\newcommand{\loss}{\mathcal{L}}

\usepackage{tcolorbox}

\newtcolorbox{questionbox}[1][]{
    colback=blue!3!white,
    colframe=black,
    fonttitle=\bfseries\small,
    title={#1},
    boxrule=0.5pt,
    left=6pt, right=6pt, top=4pt, bottom=4pt
}

\newtcolorbox{responsebox}[1][]{
    colback=white,
    colframe=black!30,
    boxrule=0.4pt,
    left=6pt, right=6pt, top=4pt, bottom=4pt
}

\usepackage{lineno}

\definecolor{darkblue}{rgb}{0, 0, 0.5}
\hypersetup{colorlinks=true, citecolor=darkblue, linkcolor=darkblue, urlcolor=darkblue}

\title{Improving Reasoning Ability via Asynchronous On-Policy Self-Distillation under Positive Rollouts}

\author{\textbf{Mingwei Xu}$^{1,~\dagger}$ and \textbf{Hao Fang}$^{1,~\dagger, *}$\\
$^1$ University of Washington, Seattle, WA 98195, USA \\
$^{~\dagger}$ equally contributed \\
$^{*}$ Authors to whom any correspondence should be addressed \\
\texttt{\{xmw03,haof459\}@uw.edu}
}

\begin{document}

\ifcolmsubmission
\linenumbers
\fi

\maketitle

\begin{abstract}
Distillation and reinforcement learning through verifiable rewards (RLVR) have achieved progress in enhancing the reasoning ability of large language models (LLMs). However, we note that negative rollouts may admit no gradation of failure severity, and the combinatorial vastness makes penalizing a few sampled negatives unlikely to cover a meaningful reward signal under sparse binary rewards. In this work, we propose \textbf{Positive-Only Policy Optimization (POPO)}, an on-policy self-distillation integrated RLVR framework in which learning occurs exclusively on online positive rollouts. Specifically, POPO utilizes bounded importance sampling over the positive rollout set. Thus, no disjoint negative rollouts are used for gradient guidance during post-training. We show that implicit negative gradients can emerge naturally through reinforcing the positive probability via rollout redistribution. Next, POPO stabilizes the policy optimization through self-distillation. First, it applies a Siamese policy network with a momentum-based adaptation law for asynchronous policy evolution. Second, we replace the KL-divergence with a bounded similarity penalty term in the Siamese representation space. We conduct extensive experiments using publicly available, well-established text-LLM models across all-level mathematical benchmarks (MATH-500, AMC23, AIME 2024/2025, and Olympiad). Our experiment demonstrates that POPO achieves superior performance compared to GRPO. Notably, we show that POPO can achieve 36.67$\%$ in AIME 2025 with Qwen-Math-7B, outperforming GRPO 30.00$\%$. Our ablation and sweep studies further illustrate the necessity and robustness. Github: \url{https://github.com/momo1443/self-distillation-rlvr-popo}.
\end{abstract}

% ============================================
% INTRODUCTION
% ============================================
\section{Introduction}
\label{sec:introduction}
\begin{figure}[h]
    \centering
    \includegraphics[width=1.0\textwidth]{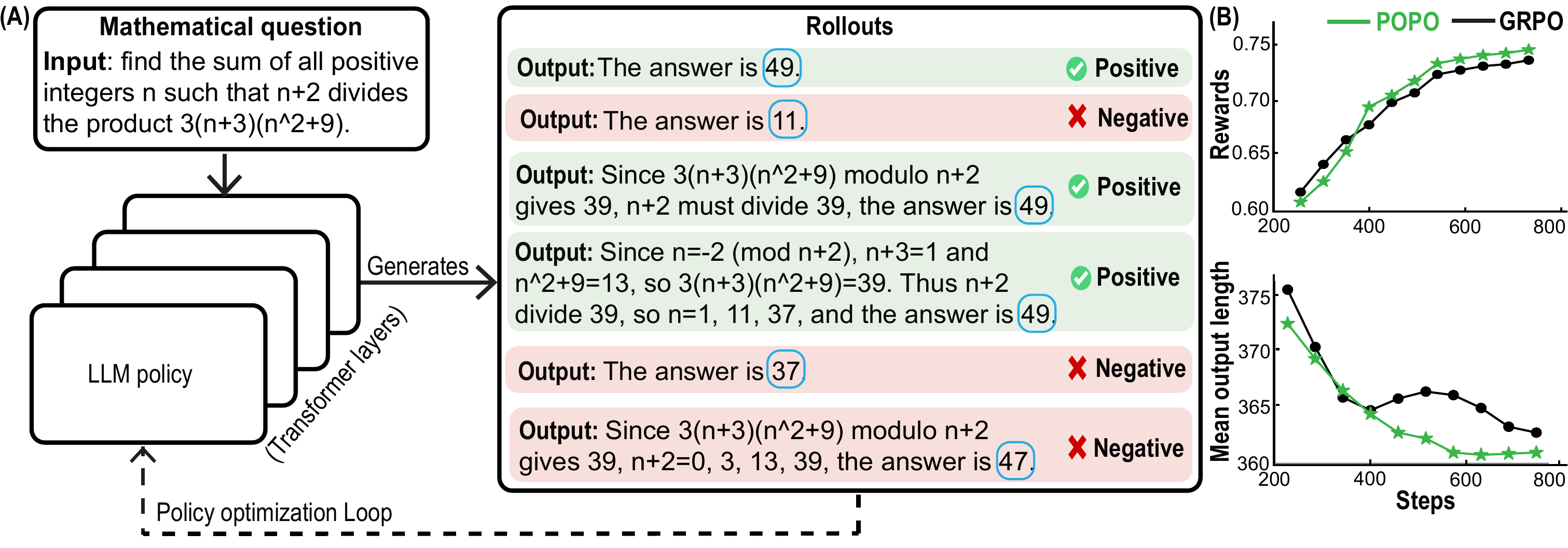}
    \caption{(A) Scheme of RLVR for mathematical question-solving tasks. A general on-policy RL algorithm will reinforce the positive rollouts, while penalizing the negative rollouts. However, the negative rollouts may not contain a meaningful reasoning chain-of-thought in addressing the graduate-level mathematical questions, which motivates the POPO algorithm design. (B) A comparison of POPO and GRPO on the reward trace (upper panel) and mean output completion length (lower panel) on a simple GSM8K dataset.}
    \label{fig: figure 1}
\end{figure}

The emergence of reasoning capabilities in post-training of large language models (LLMs) through reinforcement learning (RL) makes a significant advance in the capability enhancement and human feedback (RLHF) alignment~\citep{ouyang2022training}. Well-known work, such as DeepSeek-R1~\citep{guo2025deepseek} and OpenAI's o1 series~\citep{jaech2024openai}, demonstrates that LLMs can surge intelligent chain-of-thought (COT)~\citep{wei2022chain} reasoning ability through RL with Verifiable Rewards (RLVR)~\citep{lambert2024tulu}. In general, RLVR works by giving a sampled question $x$ and the corresponding LLM-generated response $y$ containing reasoning steps and a final answer; a verifier provides deterministic feedback~\citep{cobbe2021training}. The simplicity of the verification mechanism, together with various RL algorithms, has achieved substantial improvements in mathematical problem-solving and coding tasks~\citep{le2022coderl}, where LLMs show a strong ability of self-reflection and generate sophisticated reasoning chains.

Despite these successes, current RLVR approaches face a fundamental challenge: the treatment of incorrect responses. Take the mathematical reasoning task, for example. At the cold-start stage, the vast majority of sampled responses are incorrect, e.g., typical pass@k rates are usually below 50\% for the AMC 23 dataset, which happens due to diverse situations such as computational errors, problem misinterpretations, and inappropriate solution strategies (see figure~\ref{fig: figure 1}). Thus, we begin to reconsider and ask the question: \textit{Can a policy also improve by reinforcing only from its successes?} This seems counterintuitive at first glance, given the conventions in RL, which reinforce the policy with the good reward while penalizing it for the bad rewards. However, a few insightful observations motivate it. First, the space of incorrect responses is combinatorially vast~\citep{yue2025does} for mathematical tasks, making it unlikely that penalizing a few sampled failures meaningfully covers the failure modes. Second, on-policy distillation (OPD)~\citep{agarwal2024policy,gu2024minillm} and self-distillation~\citep{mobahi2020self} are shown to be effective as an alternative post-training approach. Third, past work in contrastive self-supervised learning (SSL) demonstrates that effective representations are feasible using only positive samples, works such as "Bootstrap your own latent" (BYOL)~\citep{grill2020bootstrap} and "Exploring simple siamese representation learning" (SimSiam)~\citep{chen2021exploring}, achieving powerful latent representations.

Inspired by the above, we propose \textbf{Positive-Only Policy Optimization (POPO)}, an asynchronous on-policy self-distillation integrated optimization approach that learns exclusively on positive reward responses. Specifically, POPO has the following contributions:
\begin{itemize}
\item POPO introduces bounded importance sampling weights, which normalize over the positive rollout set, to preferentially reinforce positive responses, while introducing implicit negative penalties.

\item POPO uses an exponential moving average (EMA) to update the siamese policy network as an adaptive policy anchor, combining the benefits from SSL and distillation insights.

\item POPO replaces KL divergence with representation-space similarity, providing a soft bounded constraint that preserves semantic structure in the predictor space.
\end{itemize}

To evaluate the POPO algorithm, we use the DeepScaleR-Preview-Dataset for policy training. Then, we conduct extensive test experiments using publicly available, well-established LLM models, including the Qwen2.5 Math family, DeepSeek-R1 distilled models, and general text-LLM models, across all-level mathematical reasoning benchmarks (MATH-500, AMC23, AIME 2024, AIME 2025, and Olympiad). Our experiment demonstrates that POPO still achieves comparable or even superior performance compared with GRPO and its variants. Especially, in AIME 2025,  we show that POPO can achieve 36.67$\%$ using Qwen-Math-7B, outperforming GRPO 30.00$\%$. Last, our ablation and sweep studies illustrate the necessity and robustness of the components and hyperparameters.

% ============================================
% Related Work
% ============================================
\section{Related Work}
\label{sec:related work}

\textbf{RLHF and RLVR.} RLHF addresses the limitation that the next-token prediction objective optimized during pretraining and supervised fine-tuning (SFT) phases does not directly capture the intentions of human users. RLHF aligns LLM behavior to be close to human intent by training a reward model on human preference and optimizing the policy via RL~\citep{christiano2017deep,stiennon2020learning}. However, learning reward modeling is an imperfect proxy for human intent and thus may be sensitive to the reward hacking issue ~\citep{gao2023scaling}. Offline preference learning, which removes the reward modeling by purely using the contrastive preference, also has widespread usage. For example, Direct Preference Optimization (DPO)~\citep{rafailov2023direct} has been applied for aligning the chatbot output to user intent~\citep{tunstall2023zephyr}. On the other hand, RLVR replaces the learnable reward modeling process with deterministic verification pipelines, where correctness can be easily assessed in domains such as mathematical reasoning and code generation~\citep{cobbe2021training,zelikman2022star}. Recent technical reports~\citep{jaech2024openai,guo2025deepseek} demonstrate that simple RLVR elicits sophisticated COT.

\textbf{Policy Optimizations for RLVR.} GRPO~\citep{shao2024deepseekmath} has become a widely used approach by removing the value network and estimating advantages through group-level reward normalization. A series of GRPO extensions have been proposed~\citep{wan2026buffer,zheng2025group,zhang2025gvpo,yue2025vapo,chen2025minimax,liu2026gdpo,xi2025bapo,hong2024orpo,meng2024simpo,cui2026clipo,yang2026learning}. Among them, Dr.GRPO~\citep{liu2025understanding} removes per-response length and standard-deviation normalization to correct an optimization bias that artificially inflates advantage estimation. Dynamic sAmpling Policy Optimization (DAPO)~\citep{yu2025dapo} introduces decoupled clipping, dynamic sampling that filters zero-advantage groups, and token-level policy gradients. However, all methods remain to learn from positive and negative rollouts.

\textbf{Learning without Negative Samples.} Learning without negative samples (or exclusively from positive examples) undergoes a long debate in the area of contrastive representation learning~\citep{liu2021self}. The early work in self-supervised learning, such as SimCLR~\citep{chen2020simple} and MoCo~\citep{he2020momentum}, needs to rely on explicit negative samples. BYOL~\citep{grill2020bootstrap} changed this paradigm by demonstrating that representations can be learned with only positive samples through an asymmetric online-target architecture. Meanwhile, SimSiam~\citep{chen2021exploring} also illustrated that a Siamese network with only a stop-gradient operation can enable meaningful latent representation learning. Recently, in the post-training of LLM, SFT works by training the alignment with giving prompts and expected solutions. Thus, it can be treated as a naive way of learning without negative samples. However, SFT-based approaches only optimize the cross-entropy objective, which suffers from limited exploration abilities. Meanwhile, SFT can also lead to distribution mismatch as the policy may evolve unboundedly, causing the strong over-fitting on the post-training data~\citep{liu2023statistical}. Recent work in Online SFT~\citep{li2025online}, Rejection Sampling Fine-Tuning~\citep{yuan2023scaling}, Self-Taught Reasoner ~\citep{zelikman2022star}, and ReST\textsuperscript{EM}~\citep{singh2023beyond} shows that iterative on-policy SFT progressively improves reasoning capabilities. Further, on-policy distillation~\citep{agarwal2024policy,gu2024minillm,mobahi2020self,yang2026learning} and on-policy self-distillation~\citep{zhao2026self} has been illustrated effectively as an alternative post-training approach, which can be viewed as an on-policy formulation of learning from positive rollouts by optimizing the student policy from the teacher policy~\citep{xu2026beyond}. 

% ===========================================
% METHODS
% ============================================
\section{Methods}
\label{sec:methods}

\subsection{Problem Setting and Preliminary}
\label{subsec:problem and preliminary}

We consider improving an LLM's mathematical reasoning ability through RLVR. Let $\dataset = \{(x_i, a_i^*)\}_{i=1}^N$ denote a set of mathematical problems, where $x_i$ is a problem statement (prompt) and $a_i^*$ is the given ground-truth answer. The policy $\pi_\theta$ stands for an LLM that generates a sequence of tokens (response) $y = (y_1, y_2, \ldots, y_T)$ given a prompt $x$: $\pi_\theta(y \mid x) = \prod_{t=1}^{T} \pi_\theta(y_t \mid x, y_{<t})$, where $y_{<t} = (y_1, \ldots, y_{t-1})$ denotes the response prefix. The reward function provides binary feedback based on answer correctness:
$R(x, y) = \indicator[\texttt{extract}(y) = a^*]$~\citep{cobbe2021training}, where $\texttt{extract}(\cdot)$ parses the final answer from the response. Here, we consider the binary reward is \emph{sparse} (provided only at sequence completion) and \emph{deterministic}. For each prompt $x$, we partition generated response rollouts into two disjoint sets
\begin{equation}
  \positives(x) = \{y : R(x, y) = 1\},  \quad  \negatives(x) = \{y : R(x, y) = 0\}. 
  \label{eq: disjoint set}  
\end{equation}
To maximize the expected reward, a general policy gradient optimization follows:  
\begin{equation}
    \max_\theta \; J(\theta) =  \E_{x \sim \dataset} \E_{y \sim \pi_\theta(\cdot \mid x)} \left[ R(x, y) \right],
    \label{eq:objective}
\end{equation}
\begin{equation}
    \nabla_\theta J(\theta) = \E_{x \sim \dataset, y \sim \pi_\theta(\cdot \mid x)}\left[A^{\pi_\theta}(x, y) \nabla_\theta \log \pi_\theta(y \mid x)\right],
    \label{eq:policy_gradient}
\end{equation}
where $A^{\pi_\theta}(x, y)$ is the estimated advantage function measuring how much better response $y$ is compared to the average responses. equation~\eqref{eq:policy_gradient} increases the probability of responses with positive advantage and decreases the negative ones. Thus, it can be viewed as $ \nabla_\theta J(\theta) = \sum_{i: A_i > 0} A_i \nabla_\theta \log \pi_\theta(y_i \mid x) + \sum_{i: A_i < 0} A_i \nabla_\theta \log \pi_\theta(y_i \mid x)$~\citep{zhu2025surprising}. In general, policy optimization methods follow:
\begin{equation}
    \loss(\theta) = -\E_{\pi_{\mathrm{anchor}}}\left[w(\theta) \cdot f(A) \cdot \log \pi_\theta(y \mid x)\right] + \lambda \cdot \mathcal{R}(\theta),
    \label{eq:unified_objective}
\end{equation}
where $w(\theta)$ is an importance sampling weight, $f(A)$ transforms the advantage, and $\mathcal{R}(\theta)$ is a regularizer term to  anchor the policy. Different methods instantiate these components differently: PPO uses clipped importance ratios with learned values for advantage estimations, whereas GRPO uses group-normalized rewards with KL divergence regularization.

\subsection{Positive-Only Policy Optimization}
\label{subsec:POPO}

\begin{figure}[t]
    \centering
    \includegraphics[width=1.0\textwidth]{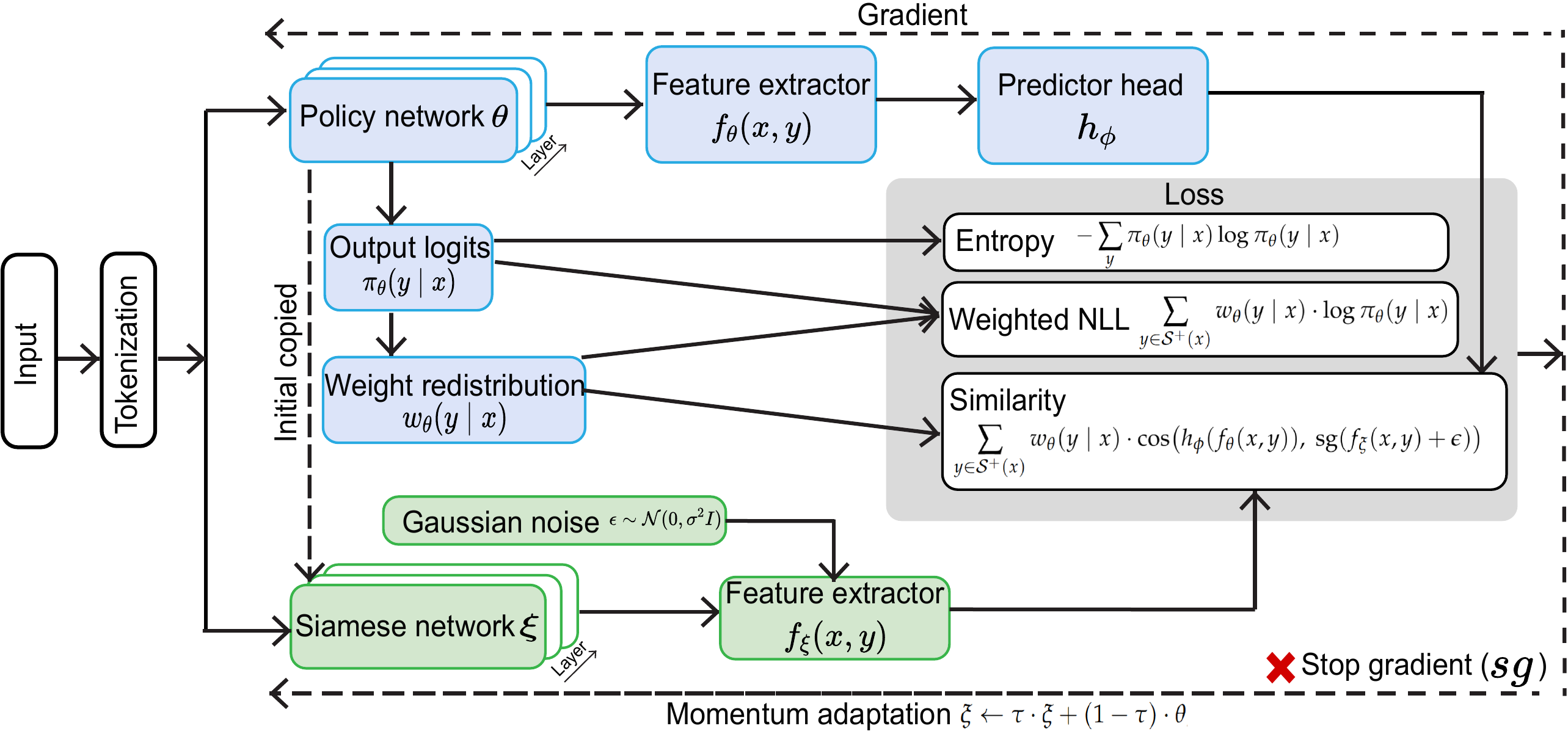}
    \caption{Overview of the POPO components. The policy network (upper part) follows the gradient backpropagation, while the Siamese network (lower part) uses momentum adaptation. Loss is in the grey box. Algorithm~\ref{alg:POPO} shows forward and backward processes.}
    \label{fig: POPO}
\end{figure}

POPO optimizes the policy using only correct responses while maintaining stability through the following mechanisms inspired by on-policy self-distillation: self-normalized weight, asynchronous EMA anchor, representation-space alignment, and entropy penalty (see figure~\ref{fig: POPO}). The overall POPO loss function is:
\begin{equation}
    \loss_{\text{POPO}}(\theta) = -\E_{x \sim \dataset} \left[ \sum_{y \in \positives(x)} w_\theta(y \mid x) \cdot \log \pi_\theta(y \mid x) \right] + \alpha \loss_{\text{sim}}(\theta, \phi, \xi) + \beta \loss_{\text{ent}}(\theta),
    \label{eq:POPO_loss}
\end{equation}
where $w_\theta(y \mid x)$ is an importance weight over the positive set; $\loss_{\text{sim}}(\theta, \phi, \xi)$ provides representation-space regularization, and $\loss_{\text{ent}}(\theta)$ is an entropy bonus encouraging exploration. The coefficients $\alpha, \beta > 0$ control regularization strength. The key distinction from standard policy gradient methods is that the expectation is taken \emph{only} over responses in $\positives(x)$ (rejected sampling as hard advantages). Although negative responses receive no explicit gradient signal, we show that the implicit negative gradient exists in section~\ref{subsec:implicit}. 

\begin{algorithm}[t]
\caption{Positive-Only Policy Optimization (POPO)}
\label{alg:POPO}
\begin{algorithmic}[1]
\REQUIRE Dataset $\dataset$, initial parameters $\theta_0$, learning rate $\eta$, EMA momentum $\tau$, \\
\hspace{2.8em} similarity weight $\alpha$, entropy weight $\beta$, group size $G$
\STATE Initialize Siamese policy network parameters: $\xi \leftarrow \theta_0$
\FOR{iteration $t = 1, 2, \ldots, T$}
    \STATE Sample batch of math question as prefix prompts $\{x_1, \ldots, x_B\}$ from $\dataset$
    \FOR{each prompt $x_i$ in batch}
        \STATE Generate $G$ responses: $\{y_1, \ldots, y_G\} \sim \pi_\theta(\cdot \mid x_i)$
        \STATE Verify correctness: $R_j = R(x_i, y_j)$ for $j = 1, \ldots, G$
        \STATE Identify positive set: $\positives(x_i) = \{y_j : R_j = 1\}$
    \ENDFOR
    \STATE Collect positive pairs: $\mathcal{B}^+ = \{(x_i, y) : y \in \positives(x_i)\}$
    \IF{$\mathcal{B}^+$ is non-empty}
        \FOR{each prompt $x$ with positives}
            \STATE Compute weights $w_\theta(y \mid x)$ via Equation~\eqref{eq:weight}
        \ENDFOR
        \STATE Compute POPO loss via Equation~\eqref{eq:POPO_loss}
        \STATE Update policy: $\theta \leftarrow \theta - \eta \nabla_\theta \loss_{\text{POPO}}$
    \ENDIF
    \STATE Update via EMA: $\xi \leftarrow \tau \xi + (1 - \tau) \theta$
\ENDFOR
\RETURN Trained policy network $\pi_\theta$
\end{algorithmic}
\end{algorithm}

\textbf{Self-Normalized Weight over Positive Set.}
POPO uses bounded and re-normalized weights over the positive set. Intuitively, this scheme creates a \emph{self-competition} situation, where responses assigned a high probability receive larger weights. Thus, we reinforce the ``confident correct answers'' more strongly while maintaining diversity through normalization. It balances between uniform weighting (when all positives are equal) and concentrated weighting (when one positive dominates).
\begin{equation}
    w_\theta(y \mid x) = \frac{\pi_\theta(y \mid x)}{Z^+(x)}, \quad \text{where} \quad Z^+(x) = \sum_{y' \in \positives(x)} \pi_\theta(y' \mid x).
    \label{eq:weight}
\end{equation}

% ========================================================================================
\textbf{Adaptive Anchor.}
Methods like BYOL~\citep{grill2020bootstrap} and SimSiam~\citep{chen2021exploring} achieve positive examples only through contrastive learning. POPO adopts this model architecture design together with LLM OPD framework~\citep{agarwal2024policy,mobahi2020self}. Basically, we maintain two copies of the policy: the policy network $\pi_\theta$ with parameters $\theta$ that we optimize, and a siamese policy $\pi_\xi$ with parameters $\xi$ that provides a stable anchor. Siamese policy parameters are updated using an exponential moving average (EMA) of the policy network parameter as a self-distillation process:
\begin{equation}
    \xi \leftarrow \tau \cdot \xi + (1 - \tau) \cdot \theta,
    \label{eq:ema}
\end{equation}
where $\tau \in (0, 1)$ is the momentum coefficient, e.g., $\tau = 0.999$.

\textbf{Representation-Space Alignment.}
With the above Siamese network architecture, we next define how the policy network should align with the Siamese network. Conventional $KL$ divergence in probability space suffers from high variance, token-level sensitivity, and rigid point-wise constraints~\citep{huang2024correcting}. Thus, our POPO aligns the hidden representation within the semantic embedding space.
\begin{equation}
    \loss_{\text{sim}}(\theta, \phi, \xi) = -\E_{x \sim \dataset}\left[\sum_{y \in \positives(x)} w_\theta(y \mid x) \cdot \mathrm{cos}\bigl(h_\phi(f_\theta(x, y)),\; \mathrm{sg}(f_\xi(x, y) + \epsilon)\bigr)\right],
    \label{eq:sim_loss}
\end{equation}
where $f_\theta(x, y)$ and $f_\xi(x, y)$ are representations extracted from the both networks, respectively; $h_\phi$ is a predictor head (implemented as MLPs), $\mathrm{sg}(\cdot)$ denotes stop-gradient, and $\epsilon \sim \mathcal{N}(0, \sigma^2 I)$ is Gaussian noise. Specifically, \textbf{Predictor head.} Predictor $h_\phi$ on the policy network is an asymmetrical design to ensure that the policy network must do ``extra work'' to match the Siamese network, preventing the collapse where both networks converge to outputs; \textbf{Stop-gradient.} The $\mathrm{sg}(\cdot)$ operator ensures gradients flow only through the policy network. The Siamese network is updated solely via EMA; \textbf{Gaussian noise.} Adding small noise $\epsilon$ provides robustness against trivial embeddings.

\textbf{Entropy Regularization.}
Like the recent approach~\citep{schulman2017proximal}, our POPO also includes the entropy loss to ensure diversity:
\begin{equation}
    \loss_{\text{ent}}(\theta) = -\E_{x \sim \dataset} \left[ H(\pi_\theta(\cdot \mid x)) \right]. \ \text{where} \ H(\pi_\theta(\cdot \mid x)) = -\sum_{y} \pi_\theta(y \mid x) \log \pi_\theta(y \mid x).
    \label{eq:entropy}
\end{equation}

% ========================================================================================
\subsection{Learning via Implicit Negative Reward}
\label{subsec:implicit}
We answer the central question: \textit{why POPO is effective without explicit penalties on the negatives?} via the weight probability redistribution, bounded divergence, and asymmetric and asynchronous self-distillation learning.
 
%  Probability Redistribution + Theorem ---
\paragraph{Implicit Negative Gradients via Weight Probability Redistribution.}
The fundamental constraint $\sum_{y} \pi_\theta(y \mid x) = 1$ implies the intuition that when POPO reinforces $\pi_\theta(y \mid x)$ for $y \in \positives(x)$, softmax normalization implicitly forces $\pi_\theta(y' \mid x)$ to decrease for $y' \in \negatives(x)$. Thus, we analyze the \textit{gradient} of $\loss_{\text{POPO}}$ at the output logit level. Let $z_y$ denote the logit (pre-softmax score) for response $y$, i.e., $\pi_\theta(y \mid x) = \exp(z_y) / \sum_{y''} \exp(z_{y''})$.
 
\begin{theorem}[Implicit Negative Gradient]
\label{thm:implicit_negative}
For any incorrect response $y' \in \negatives(x)$, the gradient of $\loss_{\text{POPO}}$~\eqref{eq:POPO_loss} with respect to the logit $z_{y'}$ satisfies
\begin{equation}
    \frac{\partial \loss_{\text{POPO}}}{\partial z_{y'}} \;=\; \pi_\theta(y' \mid x) \Big[ 1 + \beta \Big(\!\log \pi_\theta(y' \mid x) + H\!\left(\pi_\theta(\cdot \mid x)\right)\!\Big)\Big].
    \label{eq:implicit_grad}
\end{equation}
\end{theorem}
This gradient is strictly positive when $\pi_\theta(y' \mid x) \;>\; \exp\!\Big({-\tfrac{1}{\beta} - H(\pi_\theta(\cdot \mid x))}\Big)$. Simply speaking, the gradient decomposes into two reinforcing forces: a uniform ``probability tax'' from the NLL loss on positives (via the softmax partition function), and a probability-dependent penalty from entropy regularization that is strongest on the most probable incorrect responses. See the full proof in Appendix~\ref{app:proof_implicit}.
 
% Stability + Lemma ---
\paragraph{Stability through Bounded Gap and Asymmetric Architecture.}
Having established the theorem of implicit negative gradients, we next show a lemma that stabilizes the asymmetric policy training. The EMA anchoring and similarity regularization jointly prevent unbounded policy drift, avoiding the pathological gradient explosions.

\begin{lemma}[EMA-Bounded Parameter gap]
\label{lem:ema_bound}
Under POPO updates with EMA momentum $\tau \in (0, 1)$, learning rate $\eta > 0$, and gradient norm bound $\|\nabla_\theta \loss_{\text{POPO}}\| \leq G_{\max}$, the parameter divergence between the policy network $\theta_t$ and the Siamese network $\xi_t$ satisfies at steady state:
\begin{equation}
    \|\theta_t - \xi_t\| \;\leq\; \frac{\tau \, \eta \, G_{\max}}{1 - \tau}.
    \label{eq:ema_bound}
\end{equation}
\end{lemma}
This lemma guarantees the bounded gap even though we do not have the conventional $KL$ divergence (see proof in Appendix~\ref{app:proof_ema}). Further, the asymmetric network architecture ensures that the policy network must actively align with a slowly evolving average rather than collapsing to a degenerate solution. Indeed, in symmetric architectures, training only on positive risks may cause self-distill SFT collapse~\citep{li2025online}. Our POPO breaks symmetry through the above mechanisms.

\section{Experiments}
\label{sec:experiments}
% ============================================
\subsection{Experimental Setup}
\label{subsec:setup}

\textbf{Training Data and Evaluation Benchmarks.}
We adopt the \textbf{DeepScaleR-Preview-Dataset}~\citep{luo2025deepscaler} as the \textit{single training} dataset for all experiments. It contains around 40K mathematical problems sourced from diverse competitions, covering difficulty levels from elementary through olympiad-level. Using a single dataset throughout eliminates data composition as a confound when comparing benchmark algorithms. Each problem provides a ground-truth answer for binary reward verification: $R(x, y) = \indicator[\texttt{extract}(y) = a^*]$, where $\texttt{extract}(\cdot)$ parses the boxed answer from the generated response $y$ and $a^*$ is the reference answer. We evaluate on the following five benchmarks: \textbf{MATH-500}~\citep{hendrycks2021measuring}, \textbf{AMC 2023}~\citep{li2024numinamath}, \textbf{AIME 2024}~\citep{li2024numinamath}, \textbf{AIME 2025}~\citep{balunovic2025matharena}, \textbf{OlympiadBench}~\citep{he2024olympiadbench}.

\textbf{Base Models and Algorithm Baselines.}
We conduct full-parameter fine-tuning across six base models spanning two scales ($\sim$1.5B and $\sim$7--8B) and three pretraining paradigms.
\textbf{Math-specialized}: Qwen2.5-Math-1.5B and Qwen2.5-Math-7B~\citep{yang2024qwen2}, pretrained on large-scale mathematical corpora, representing the strongest available math foundations.
\textbf{Distilled-reasoning}: DeepSeek-R1-Distill-Qwen-1.5B (named as R1-Distill-1.5B) and DeepSeek-R1-Distill-Qwen-7B (named as R1-Distill-7B)~\citep{guo2025deepseek} , distilled from DeepSeek-R1's COT reasoning, thus with strong reasoning priors.
\textbf{General-purpose}: Llama-3.1-8B~\citep{grattafiori2024llama} and DeepSeek-Math-7B~\citep{shao2024deepseekmath}, included as controls for the general models. All methods are implemented using the \textit{Hugging Face TRL} package for consistent reward computation, training, and generation. We compare six methods: \textbf{Base}, \textbf{SFT}, \textbf{GRPO}~\citep{shao2024deepseekmath}, \textbf{Dr. GRPO}~\citep{liu2025understanding}, \textbf{DAPO}~\citep{yu2025dapo}, \textbf{POPO} (our approach). For the two general-purpose base models (Llama-3.1-8B and DeepSeek-Math-7B), we compare only the GRPO and POPO.

\textbf{Evaluation Protocol.}
We consistently report the Pass@8 metrics for all benchmarks. Pass@$k$ is computed using the unbiased estimator~\citep{chen2021evaluating,lewkowycz2022solving,wang2022self}. Output responses use vLLM~\citep{kwon2023efficient} for fast generation. Details about the experimental setup are attached in Appendix~\ref{app:training_details}.

% We use the Wilcoxon signed-rank test~\citep{wilcoxon1945individual} to evaluate the significance.

% ============================================
\subsection{POPO Achieved Superior Performance When the Reasoning Task Became Harder}
\label{subsec:main_results}
% ---------- Main Table: Pass@8 ----------
\begin{table}[h]
\centering
\scriptsize
\resizebox{\textwidth}{!}{%
\begin{tabular}{@{}llcccccc@{}}
\toprule
Models & Algorithms & \textbf{MATH-500} & \textbf{AMC 23} & \textbf{AIME 24} & \textbf{AIME 25} & \textbf{Olympiad} & \textbf{Avg.} \\
\midrule
\multirow{6}{*}{\small Qwen-Math-1.5B}
& Base         & 80.20 & 65.00 & 20.00 & 6.67 & 48.22 & 44.02 \\
& SFT         & 78.80 & 71.89 & 20.70 & 10.00 & 47.77 & 45.83 \\
& GRPO         & 86.20 & 75.00 & 23.33 & 16.25 & 50.30 & 50.22 \\
% & SAPO         & 86.20 & 76.17 & 23.33 & 17.11 & 50.74 &  50.71 \\
& Dr. GRPO      & 86.00 & 76.62 & 25.18 & 17.39 & 50.89 & 51.22 \\
& DAPO         & 86.40 & 75.28 & 22.85 & 16.67 & 50.59 & 50.36 \\
& \textbf{POPO} & \textbf{86.60} & \textbf{77.50} & \textbf{26.67} & \textbf{23.33} & \textbf{51.19} & \textbf{53.06} \\
\midrule
\multirow{6}{*}{\small R1-Distill-1.5B}
& Base         & 86.20 & 75.00 & 20.00 & 20.00 & 54.15 & 51.07 \\
& SFT         & 87.00 & 77.08 & 21.07 & 21.19 & 55.04 & 52.28 \\
& GRPO         & 90.20 & 80.00 & 30.00 & 26.32 & 58.61 & 57.03\\
% & SAPO         & 90.60 & 83.32 & 30.21 & 26.67 & \textbf{59.20} & 58.00 \\
& Dr. GRPO      & 90.40 & 83.46 & 33.33 & 25.84 & 58.90 & 58.39 \\
& DAPO         & 90.60 & 82.75 & 29.72 & 26.45 & 58.75 & 57.65 \\
& \textbf{POPO} & \textbf{90.80} & \textbf{85.24} & \textbf{36.67} & \textbf{27.86} & \textbf{59.05} & \textbf{59.92} \\
\midrule
\multirow{6}{*}{\small Qwen-Math-7B}
& Base         & 82.60 & 72.50 & 33.33 & 20.00 & 44.21 & 50.53 \\
& SFT         & 84.00 & 70.55 & 32.81 & 18.93 & 44.07 & 50.07 \\
& GRPO         & 90.80 & 85.50 & 43.33 & 30.00 & 56.38 & 61.20 \\
% & SAPO         & 91.20 & \textbf{86.12} & \textbf{46.67} & 31.04 & 56.67 & 62.34 \\
& Dr. GRPO      & 91.00 & 85.63 & \textbf{46.45} & 30.99 & 56.82 & 62.18 \\
& DAPO         & \textbf{91.40} & \textbf{85.97} & 44.76 & 32.31 & 56.53 & 62.19 \\
& \textbf{POPO} & 90.80 & 85.75 & 45.13 & \textbf{36.67} & \textbf{57.27} & \textbf{63.12} \\
\midrule
\multirow{6}{*}{\small R1-Distill-7B}
& Base         & 90.40 & 82.50 & 36.67 & 26.67 & 63.65 & 59.98 \\
& SFT         & 91.40 & 84.82 & 39.15 & 33.71 & 62.31 & 62.28 \\
& GRPO         & 93.60 & 87.50 & 46.67 & 33.33 & 66.91 & 65.60 \\
% & SAPO         &\textbf{94.20} & \textbf{88.21} & 46.71 & 35.00 & 67.06 & \textbf{66.24} \\
& Dr. GRPO      & 93.80 & 87.67 & \textbf{48.59} & 33.50 & \textbf{67.51} & 66.21 \\
& DAPO         & \textbf{94.00} & \textbf{87.75} & 47.83  & 33.87 & 66.28 & 65.95 \\
& \textbf{POPO} & 93.20 & 87.25 & 47.22 & \textbf{36.67} & 66.76 & \textbf{66.22} \\
\bottomrule
\end{tabular}%
}
\caption{Comparison of RL algorithms' performance across mathematical reasoning benchmarks using math-specialized and distilled base models. Bold numbers indicate the best.}
\label{tab:main_passk}
\end{table}

We began our experimental analysis with a small yet effective Qwen-Math-1.5B model, which has been widely used in recent studies~\citep{wang2025reinforcement}. In Table~\ref{tab:main_passk} first row, we observed that the Qwen-Math-1.5B model achieved superior performance across all five mathematical benchmarks (Math-500: $86.60$, AMC 23: $77.50$, AIME 24: $26.67$, AIME 25: $23.33$, Olympiad: $51.19$). Further, compared to the other baseline RL algorithms, our POPO performed in average better across five datasets (POPO: $53.06$, v.s. GRPO $50.22$; Dr. GRPO $51.22$; DAPO $50.36$). These results gave the first confidence that our proposed POPO also stabilized reasonable policy optimization using only online positive rollouts. 

Next, we further tested our POPO using R1-Distill-1.5B, Qwen-Math-7B (see Table~\ref{tab:main_passk}, second and third row). These widely used open-source models are specialized in mathematical reasoning tasks. First, as consistently investigated in many prior research~\citep{snell2024scaling}, both the distilled model (from powerful DeepSeek-R1 models with strong COT ability) and large models showed better performance in general. For example, both POPO with R1-Distill-1.5B ($59.92$) and POPO with Qwen-Math-7B ($63.12$) outperformed POPO with Qwen-Math-1.5B ($53.06$). Second, across all math-specialized and distilled base models, we observed that POPO almost performed better on the hard mathematical reasoning tasks, e.g., AIME 25 36.67 and Olympiad 57.27 using Qwen-Math-7B. Considering the fact that AIME25 was released (publicly available) after the model release, it rejected the potential that POPO performed better due to any potential data leakage in the pretraining phase. Meanwhile, the clear trend that POPO started with comparable performance with GRPO in MATH-500, and later achieved better
performance in the AIME 25 (AIME 25 $17.59\%$ v.s. MATH-500 $1.66\%$), suggested relative $15.93\%$ improvement. This trend illustrated our assumption that the combinatorial vastness made penalizing a few sampled negatives unlikely to cover a meaningful reward signal under sparse binary rewards. We provided a concrete example of different RL responses given the same AIME 25 mathematical question in Appendix~\ref{Examples responses from POPO and other RL algorithms}.

Regarding the R1-Distill-7B model in Table~\ref{tab:main_passk} fourth row, while our POPO still outperformed the baseline results (POPO: 66.22 v.s. Base: 59.98), our POPO slightly underperformed the GRPO across different datasets. One potential cause was that the model entropy reduced drastically when a larger 7B model was distilled using DeepSeek-R1. This might limit the on-policy exploration ability, while the negative results would contain more reasonable penalty information for GRPO (and its extensions) for efficient policy optimization. One of our future works will further examine the negative rollout effects within the proposed asynchronous self-distillation RLVR framework.

\subsection{Ablation and Sweeping Studies}
\label{subsec:ablation}

\begin{table}[h]
\centering
\small
\begin{tabular}{@{}l c c c@{}}
\toprule
& \multicolumn{1}{c}{\textbf{AIME 2025}} & \multicolumn{1}{c}{\textbf{Olympiad}} & \multicolumn{1}{c}{\textbf{Avg.}} \\
\midrule
POPO \ (default)                     & 23.33 & 51.19 & 37.26 \\
POPO \ w/o negative rollouts; w uniform weight   & 16.67 & 49.11 & 32.89 \\
POPO \ w negative rollouts; w/o uniform weight     & 13.33 & 48.07 & 30.70 \\
POPO \ w negative rollouts; w uniform weight       & 6.67 & 39.32 & 23.00  \\
\bottomrule
\end{tabular}
\caption{Ablation study on importance sampling and weight redistribution. The default setup is without (w/o) negative rollouts and with weight redistribution. w and w/o stand for with and without, respectively.}
\label{tab:ablation 1}
\end{table}

Having established the results in Table~\ref{tab:main_passk}, we next investigated several key components in our POPO, i.e., importance sampling, weight redistribution, and momentum adaptation, and representation-space alignment via ablation and sweeping studies. Due to the limit of computational resources, we took Qwen-Math-1.5B as our base model for simplicity.

\textbf{Importance sampling and weight redistribution.} We first considered the situation where we removed the positive rollout constraints, indicating the same rollout usage as GRPO. We surprisingly found that POPO still maintained reasonable performance (30.70 v.s. default 37.26) across AIME 25 and Olympiad benchmarks, which contradicted our intuition. Recall that in the GRPO, we have grouped relative advantage estimation to penalize the negative rollouts. In contrast, POPO respects all rollouts for positive reinforcement. We assumed this magic was due to the weight redistribution, which guaranteed the negative gradient, by implicitly penalizing the negative rollouts. In the ablation, we removed the redistribution weight by treating all rollouts uniformly. In Table~\ref{tab:ablation 1}, we observed that POPO's performance dropped drastically (23.00). \textbf{Momentum adaptation and alignment.} In Table~\ref{tab:ablation 2}, we observed that by removing the momentum adaptation, the POPO performances dropped to 30.85, indicating the necessity of slowing the adaptation of the anchor model. This observation also aligned with recent studies that a static reference model may limit the post-training RL performance. Table~\ref{tab:ablation 2} second row showed that without the representation-space alignment, the policy training was unstable, causing the performance drop. \textbf{Sweep study.} We presented summarized sweep studies results for predictor layers, predictor dimensions, and momentum rate in Table~\ref{tab:Sweep study}.

\begin{table}[t]
\centering
\small
\begin{tabular}{@{}l c c c@{}}
\toprule
& \multicolumn{1}{c}{\textbf{AIME 2025}} & \multicolumn{1}{c}{\textbf{Olympiad}} & \multicolumn{1}{c}{\textbf{Avg.}} \\
\midrule
POPO        & 23.33 & 51.19 & 37.26 \\
POPO \ w/o momentum   & 13.33 & 48.37 & 30.85 \\
POPO \ w/o alignment   &  13.33 & 50.15 & 31.74 \\
\bottomrule
\end{tabular}
\caption{Ablation study on momentum adaptation and representation-space alignment.}
\label{tab:ablation 2}
\end{table}

\begin{table}[h]
\centering
\small
\begin{tabular}{@{}ll@{}}
\toprule
\textbf{Parameter} & \textbf{Avg.} \\
\midrule
Predictor layers = 2      & 53.06  \\
Predictor layers = 3      & 52.45  \\
Predictor layers = 4     & 51.97 \\
Predictor dimension = 1024       & 52.38   \\
Predictor dimension = 2048       & 52.64  \\
Predictor dimension = 4096       & 53.06 \\
Momentum = 0.99       & 52.13   \\
Momentum = 0.999       & 53.06   \\
Momentum = 0.9999       & 52.21   \\
\bottomrule
\end{tabular}
\caption{Sweep study on predictor layer, embedding dimension, and momentum rate. The Avg. is calculated for all five mathematical datasets.}
\label{tab:Sweep study}
\end{table}

\subsection{POPO for General Text-LLM Models}
Last, we tested the POPO performance using Llama-3.1-8B and DeepSeek-Math-7B. Both models differ from the major Qwen family in the pretraining phase, thus further providing model-agnostic evidence for applying POPO to improve reasoning ability. Specifically, Llama-3.1-8B is a general-purpose language model without any math-specific data curation, and DeepSeek-Math-7B is a continuous pretraining from DeepSeek-Coder-Base-v1.5 7B~\citep{guo2024deepseek}. In Table~\ref{tab:Comparison of RL algorithms' performance using general text-LLM models}, we also observed that POPO in general performs comparably to GRPO (Llama-3.1-8B: 29.84 v.s. 26.72; DeepSeek-Math-7B: 50.15 v.s. 47.81), which again demonstrated the potential model-agnostic advantages of POPO.

\begin{table}[h]
\centering
\resizebox{\textwidth}{!}{%
\begin{tabular}{@{}ll cccccc@{}}
\toprule
Models& Algorithms & \textbf{MATH-500} & \textbf{AMC 23} & \textbf{AIME 24} & \textbf{AIME 25} & \textbf{Olympiad} & \textbf{Avg.} \\
\midrule
\multirow{2}{*}{\small Llama-3.1-8B}
& GRPO         & 56.40 & 35.22 & 10.00 & 6.67 & 25.32 & 26.72 \\
& \textbf{POPO} & \textbf{59.80} & \textbf{37.78} & \textbf{13.33} & \textbf{10.00} & \textbf{28.29} & \textbf{29.84} \\
\midrule
\multirow{2}{*}{\small DeepSeek-Math-7B}
& GRPO         & 80.20 & 73.38 & 20.00 & 16.67 & 48.81 & 47.81 \\
& \textbf{POPO} & \textbf{82.60} & \textbf{74.53} & \textbf{23.33} & \textbf{20.00} & \textbf{50.29} & \textbf{50.15} \\
\bottomrule
\end{tabular}%
}
\caption{Performance comparison with GRPO using general text-LLM models.}
\label{tab:Comparison of RL algorithms' performance using general text-LLM models}
\end{table}

\section{Discussion and Conclusion}
We propose the POPO framework that enables effective on-policy self-distillation optimization on positive only rollouts. POPO is unique in combining innovative components, including bounded importance sampling, adaptive EMA anchoring, and representation-space alignment. Further, we also provide mathematical proofs beyond the heuristic intuition. We test POPO using well-established open-source base LLM models on public mathematical reasoning benchmarks. Our results illustrate the potential implications of POPO in future RLVR and LLM reasoning research. Our work also has a few limitations. First, we design POPO for sparse binary rewards, yet many practical scenarios are directly involved with dense reward signals~\citep{setlur2024rewarding,cui2025process}. Extending POPO under the dense reward signals setup will be interesting. Second, POPO is validated only on text-only LLMs with purely mathematical reasoning benchmarks. We will extend POPO to other domains such as code generation tasks~\citep{le2022coderl}, agentic RL~\citep{feng2024agile}, and multimodal reasoning~\citep{lu2023mathvista}, and generative modeling~\citep{cui2026joint}. Third, we only test POPO up to 7B models. Future work will scale it up to 32B or 70B~\citep{yang2024qwen2}.

\section*{Disclose Use of LLMs}
We used LLMs, such as Claude and GPT, for initial idea brainstorming and later to check grammar and spelling and improve the clarity of our manuscript's language. All scientific contents, including figures, tables, experimental design, analysis, and interpretations, are entirely the original work of the authors.

\section*{Ethics Statement}
Our research does not involve any potential conflicts of interest or sponsorship issues. We have carefully considered and addressed issues related to discrimination, bias, and fairness in our methodology. This study does not involve any privacy or security concerns, fully complies with legal requirements, and upholds the highest standards of research integrity. All experimental procedures and data processing follow established ethical guidelines for machine learning research.

\section*{Reproducibility Statement}
We provide implementation details in Appendix~\ref{app:training_details}, where all experimental settings, hyperparameters, and dataset specifications are documented. A pre-release version is in \url{https://github.com/momo1443/self-distillation-rlvr-popo}. The complete code will be released upon acceptance.

% \section*{Acknowledgments}
% \section*{Author Contributions}

\bibliography{colm2026_conference}
\bibliographystyle{colm2026_conference}

\appendix
\section*{Appendix}
\setcounter{section}{0}
\setcounter{equation}{0}
\setcounter{table}{0}
\setcounter{figure}{0}

\section{Training and Evaluation Benchmarks}
We train on DeepScaleR-Preview-Dataset and evaluate on five benchmarks spanning elementary through Olympiad-level difficulty.
\begin{itemize}
\item{DeepScaleR-Preview-Dataset}~\citep{luo2025deepscaler}:  a curated training corpus of approximately 40,000 unique mathematics problem--answer pairs compiled from AIME problems (1984--2023), AMC problems (before 2023), the Omni-MATH dataset, and the Still dataset, spanning difficulty levels from high-school competition to advanced mathematical reasoning.
\item{MATH-500}~\citep{hendrycks2021measuring}: a 500-problem representative subset of the MATH test set, from the MATH benchmark that OpenAI created in their \textit{Let's Verify Step by Step}~\citep{jaech2024openai} paper. 
\item{AMC 2023}~\citep{li2024numinamath}:  a 40-problem benchmark curated from the 2023 AMC 12A and AMC 12B contests, covering short-answer high-school competition mathematics.
\item{AIME 2024}~\citep{li2024numinamath}: a 30-problem benchmark drawn from the 2024 AIME I and AIME II exams, consisting of challenging high-school competition problems with numerical final answers.
\item{AIME 2025}~\citep{balunovic2025matharena}: a 30-problem benchmark drawn from the 2025 AIME I and AIME II exams; because it was released in 2025, it is commonly used as a fresher benchmark with lower contamination risk for models whose pretraining predates the contest.
\item{OlympiadBench}~\citep{he2024olympiadbench}: the English mathematics portion of OlympiadBench, an Olympiad-level bilingual multimodal scientific benchmark; the official release provides English-math subsets spanning text-only and multimodal problems, with both open-ended and proof-based variants.
\end{itemize}

The downloading link for each dataset can be accessed from the Hugging Face website:
\begin{itemize}

\item{DeepScaleR-Preview-Dataset}: \url{https://huggingface.co/datasets/agentica-org/DeepScaleR-Preview-Dataset}
\item{MATH-500}: \url{https://huggingface.co/datasets/HuggingFaceH4/MATH-500}
\item{AMC 2023}: \url{https://huggingface.co/datasets/math-ai/amc23}
\item{AIME 2024}: \url{https://huggingface.co/datasets/Maxwell-Jia/AIME_2024}
\item{AIME 2025}: \url{https://huggingface.co/datasets/MathArena/aime_2025}
\item{OlympiadBench}: \url{https://huggingface.co/datasets/Hothan/OlympiadBench}
\end{itemize}

\section{Proofs of Theoretical Results}
\label{app:proofs}
 
\subsection{Proof of Theorem~\ref{thm:implicit_negative}: Implicit Negative Gradient}
\label{app:proof_implicit}

\begin{proof}
We analyze the gradient of $\loss_{\text{POPO}}$ with respect to the logit $z_{y'}$ for an incorrect response $y' \in \negatives(x)$. Recall the POPO loss in the main text equation~\eqref{eq:POPO_loss}:
\begin{equation}
    \loss_{\text{POPO}}(\theta) = \underbrace{-\sum_{y \in \positives(x)} w_\theta(y \mid x) \log \pi_\theta(y \mid x)}_{\loss_{\text{NLL}}} + \alpha \cdot \loss_{\text{sim}}(\theta, \phi, \xi) + \beta \cdot \loss_{\text{ent}}(\theta).
\end{equation}
 
We compute the gradient contribution of each component separately. Throughout, we work with a single prompt $x$ and drop the conditioning for brevity.
 
\paragraph{Step 1: Softmax preliminary.}
Under the softmax parameterization $\pi_\theta(y \mid x) = \exp(z_y)/\sum_{y''}\exp(z_{y''})$, the standard softmax Jacobian gives
\begin{align}
    \frac{\partial \pi_\theta(y \mid x)}{\partial z_{y'}} &= \pi_\theta(y' \mid x)\bigl(\delta_{y,y'} - \pi_\theta(y \mid x)\bigr),
    \label{eq:softmax_jac} \\
    \frac{\partial \log \pi_\theta(y \mid x)}{\partial z_{y'}} &= \delta_{y,y'} - \pi_\theta(y' \mid x),
    \label{eq:log_softmax_deriv}
\end{align}
where $\delta_{y,y'}$ is the Kronecker delta. Since $y \in \positives(x)$ and $y' \in \negatives(x)$ are in disjoint sets, we have $\delta_{y,y'} = 0$, and thus
\begin{equation}
    \frac{\partial \log \pi_\theta(y \mid x)}{\partial z_{y'}} = -\pi_\theta(y' \mid x) \quad \text{for all } y \in \positives(x).
    \label{eq:cross_deriv}
\end{equation}
 
\paragraph{Step 2: Weight invariance for negative logits.}
The bounded importance weights: $w_\theta(y \mid x) = \pi_\theta(y \mid x)/Z^+(x)$ where $Z^+(x) = \sum_{y'' \in \positives(x)} \pi_\theta(y'' \mid x)$. We show that these weights are invariant to perturbations of negative logits. For $y \in \positives(x)$ and $y' \in \negatives(x)$, by the quotient rule:
\begin{equation}
    \frac{\partial w_\theta(y \mid x)}{\partial z_{y'}} = \frac{1}{Z^+}\frac{\partial \pi_\theta(y \mid x)}{\partial z_{y'}} - \frac{\pi_\theta(y \mid x)}{(Z^+)^2}\frac{\partial Z^+}{\partial z_{y'}}.
    \label{eq:weight_deriv_quotient}
\end{equation}
From~\eqref{eq:softmax_jac} with $y \neq y'$:
\begin{equation}
    \frac{\partial \pi_\theta(y \mid x)}{\partial z_{y'}} = -\pi_\theta(y \mid x)\,\pi_\theta(y' \mid x).
\end{equation}
Summing over the positive set:
\begin{equation}
    \frac{\partial Z^+}{\partial z_{y'}} = \sum_{y'' \in \positives(x)} \frac{\partial \pi_\theta(y'' \mid x)}{\partial z_{y'}} = -\pi_\theta(y' \mid x) \sum_{y'' \in \positives(x)} \pi_\theta(y'' \mid x) = -\pi_\theta(y' \mid x)\, Z^+.
\end{equation}
Substituting back into~\eqref{eq:weight_deriv_quotient}:
\begin{align}
    \frac{\partial w_\theta(y \mid x)}{\partial z_{y'}} &= \frac{-\pi_\theta(y \mid x)\,\pi_\theta(y' \mid x)}{Z^+} - \frac{\pi_\theta(y \mid x)}{(Z^+)^2}\bigl(-\pi_\theta(y' \mid x)\,Z^+\bigr) \notag \\
    &= \frac{-\pi_\theta(y \mid x)\,\pi_\theta(y' \mid x)}{Z^+} + \frac{\pi_\theta(y \mid x)\,\pi_\theta(y' \mid x)}{Z^+} \;=\; 0.
    \label{eq:weight_invariance}
\end{align}
This invariance has a clear interpretation: rescaling all probabilities by the same factor (as the softmax does when a single logit changes) leaves the \emph{ratio} $\pi_\theta(y \mid x)/Z^+$ unchanged. Crucially, this means the weight normalization over $\positives(x)$ does not attenuate the implicit gradient on negatives.
 
\paragraph{Step 3: NLL gradient on negative logits.}
Applying the product rule and using~\eqref{eq:weight_invariance}:
\begin{align}
    \frac{\partial \loss_{\text{NLL}}}{\partial z_{y'}} &= -\sum_{y \in \positives(x)} \left[\frac{\partial w_\theta(y \mid x)}{\partial z_{y'}} \cdot \log \pi_\theta(y \mid x) + w_\theta(y \mid x) \cdot \frac{\partial \log \pi_\theta(y \mid x)}{\partial z_{y'}}\right] \notag \\
    &= -\sum_{y \in \positives(x)} w_\theta(y \mid x) \cdot \bigl(-\pi_\theta(y' \mid x)\bigr) \notag \\
    &= \pi_\theta(y' \mid x) \sum_{y \in \positives(x)} w_\theta(y \mid x) \;=\; \pi_\theta(y' \mid x),
    \label{eq:nll_on_neg}
\end{align}
where the last equality uses $\sum_{y \in \positives(x)} w_\theta(y \mid x) = 1$. This term is unconditionally positive: gradient descent on the NLL loss alone already decreases the probability of every incorrect response.
 
\paragraph{Step 4: Entropy gradient on negative logits.}
The entropy loss is $\loss_{\text{ent}} = -H(\pi_\theta(\cdot \mid x)) = \sum_{y} \pi_\theta(y \mid x) \log \pi_\theta(y \mid x)$. Its gradient with respect to $z_{y'}$ is:
\begin{align}
    \frac{\partial \loss_{\text{ent}}}{\partial z_{y'}} &= \sum_{y} \frac{\partial \pi_\theta(y \mid x)}{\partial z_{y'}} \bigl(1 + \log \pi_\theta(y \mid x)\bigr) \notag \\
    &= \sum_{y} \pi_\theta(y' \mid x)\bigl(\delta_{y,y'} - \pi_\theta(y \mid x)\bigr)\bigl(1 + \log \pi_\theta(y \mid x)\bigr) \notag \\
    &= \pi_\theta(y' \mid x)\Bigl[\bigl(1 + \log \pi_\theta(y' \mid x)\bigr) - \underbrace{\sum_{y} \pi_\theta(y \mid x)\bigl(1 + \log \pi_\theta(y \mid x)\bigr)}_{= \;1 - H(\pi_\theta(\cdot \mid x))}\Bigr] \notag \\
    &= \pi_\theta(y' \mid x)\bigl[\log \pi_\theta(y' \mid x) + H(\pi_\theta(\cdot \mid x))\bigr].
    \label{eq:ent_on_neg}
\end{align}
 
\paragraph{Step 5: Similarity loss does not contribute at the logit level.}
The similarity loss in the main text equation~\eqref{eq:sim_loss} depends on the negative logit $z_{y'}$ only through the weights $w_\theta(y \mid x)$. By Step~2, $\partial w_\theta(y \mid x)/\partial z_{y'} = 0$, and the cosine similarity term involves hidden representations $f_\theta(x, y)$ for $y \in \positives(x)$, which are computed from intermediate layers and do not depend on the output logit $z_{y'}$. Therefore, $\partial \loss_{\text{sim}} / \partial z_{y'} = 0$.
 
\paragraph{Step 6: Combining and sign analysis.}
Summing the contributions from Steps~3--5:
\begin{align}
    \frac{\partial \loss_{\text{POPO}}}{\partial z_{y'}} &= \frac{\partial \loss_{\text{NLL}}}{\partial z_{y'}} + \beta \frac{\partial \loss_{\text{ent}}}{\partial z_{y'}} + \alpha \frac{\partial \loss_{\text{sim}}}{\partial z_{y'}} \notag \\
    &= \pi_\theta(y' \mid x) + \beta\,\pi_\theta(y' \mid x)\bigl[\log \pi_\theta(y' \mid x) + H(\pi_\theta(\cdot \mid x))\bigr] + 0 \notag \\
    &= \pi_\theta(y' \mid x)\Big[1 + \beta\bigl(\log \pi_\theta(y' \mid x) + H(\pi_\theta(\cdot \mid x))\bigr)\Big].
    \label{eq:combined}
\end{align}
Since $\pi_\theta(y' \mid x) > 0$, the sign is determined by the bracketed factor. Setting it to zero:
\begin{equation}
    1 + \beta\bigl(\log \pi_\theta(y' \mid x) + H\bigr) = 0 \;\;\Longleftrightarrow\;\; \pi_\theta(y' \mid x) = \exp\!\Big({-\tfrac{1}{\beta} - H}\Big).
\end{equation}
The gradient is strictly positive whenever $\pi_\theta(y' \mid x)$ exceeds this threshold. In practice, this condition is typically satisfied (as $\exp\!\Big({-\tfrac{1}{\beta} - H}\Big)$ is close to $0$) for any response with non-negligible probability, which completes the proof.
\end{proof} 

The proof reveals two reinforcing forces that drive the implicit negative gradient. The first force originates from the weighted NLL loss on positives: through the softmax partition function, reinforcing correct responses imposes a uniform ``probability tax'' of $\pi_\theta(y' \mid x)$ on every incorrect logit. Notably, the bounded importance weights $w_\theta(y \mid x)$ do not attenuate this force (as shown in Step~2, $\partial w_\theta(y \mid x) / \partial z_{y'} = 0$ for all $y' \in \negatives(x)$), so the weight normalization over $\positives(x)$ passes through without affecting the implicit gradient on negatives. The second force arises from the entropy regularization equation~\eqref{eq:ent_on_neg}, which adds the strongest suppression (as the penalty) on the most probable incorrect responses.

% ========================================================================================
\subsection{Proof of Lemma~\ref{lem:ema_bound} (EMA-Bounded Parameter gap)}
\label{app:proof_ema}
 
\begin{proof}
We define the parameter gap $\delta_t = \theta_t - \xi_t$, which helps to build a recurrence relation using the POPO update rules.
 
\paragraph{Step 1: Parameter gap recurrence.}
At iteration $t$, the online parameters are updated via gradient descent:
\begin{equation}
    \theta_{t+1} = \theta_t - \eta\, g_t, \quad \text{where } g_t = \nabla_\theta \loss_{\text{POPO}}(\theta_t) \text{ with } \|g_t\| \leq G_{\max}.
    \label{eq:online_update}
\end{equation}
The target parameters are updated via EMA~\eqref{eq:ema}:
\begin{equation}
    \xi_{t+1} = \tau\, \xi_t + (1 - \tau)\, \theta_{t+1}.
    \label{eq:target_update}
\end{equation}
Subtracting~\eqref{eq:target_update} from~\eqref{eq:online_update}:
\begin{align}
    \delta_{t+1} &= \theta_{t+1} - \xi_{t+1} = \theta_{t+1} - \tau\,\xi_t - (1-\tau)\,\theta_{t+1} \notag \\
    &= \tau\,\theta_{t+1} - \tau\,\xi_t = \tau\bigl(\theta_{t+1} - \xi_t\bigr) \notag \\
    &= \tau\bigl(\theta_t - \eta\, g_t - \xi_t\bigr) = \tau\,\delta_t - \tau\,\eta\, g_t.
    \label{eq:gap_recurrence}
\end{align}
 
\paragraph{Step 2: Norm bound via contraction.}
Taking norms in~\eqref{eq:gap_recurrence} and applying the triangle inequality:
\begin{equation}
    \|\delta_{t+1}\| \leq \tau\,\|\delta_t\| + \tau\,\eta\, G_{\max}.
    \label{eq:norm_recurrence}
\end{equation}
This is a contractive recurrence with rate $\tau < 1$ and constant driving term $\tau\,\eta\, G_{\max}$.
 
\paragraph{Step 3: Steady-state bound.}
Starting from $\delta_0 = \theta_0 - \xi_0 = 0$ (since $\xi_0 = \theta_0$), we unroll~\eqref{eq:norm_recurrence}:
\begin{equation}
    \|\delta_t\| \leq \tau\,\eta\, G_{\max} \sum_{k=0}^{t-1} \tau^k = \tau\,\eta\, G_{\max} \cdot \frac{1 - \tau^t}{1 - \tau}.
\end{equation}
Taking $t \to \infty$:
\begin{equation}
    \|\delta_t\| \leq \frac{\tau\,\eta\, G_{\max}}{1 - \tau},
    \label{eq:steady_state}
\end{equation}
which establishes the claimed bound~\eqref{eq:ema_bound}. As a concrete example, for $\tau = 0.999$, $\eta = 1 \times 10^{-6}$, and $G_{\max} = 1$, the bound gives $\|\delta\| <  10^{-3}$, confirming typical bounded gap throughout training.
\end{proof}
 
Lemma~\ref{lem:ema_bound} bounds the parameter gap. In practice, the similarity loss $\loss_{\text{sim}}$ also provides a second restoring force: its gradient component $\alpha \nabla_\theta \loss_{\text{sim}}$ pushes $\theta$ back toward $\xi$ in representation space, further reducing the effective divergence. The two mechanisms are complementary. When $\tau$ is close to $1$ (strong momentum anchoring), the EMA alone provides tight control and $\alpha$ can be set small. Conversely, when $\tau$ is smaller (faster target adaptation), a larger $\alpha$ compensates by strengthening the representation-space tether.

% ============================================
% APPENDIX: TRAINING DETAILS
% ============================================
\section{Training Details}
\label{app:training_details}

\paragraph{Infrastructure.}
All experiments are conducted on a single Kubernetes (K8s) node equipped with 4$\times$ NVIDIA A100 80GB GPUs.
We use PyTorch with Fully Sharded Data Parallelism (FSDP) for distributed training and vLLM (server mode) for rollout generation. It takes around 11 hours to run POPO for 1 epoch training on 1.5B models and 24 hours on 7B models, respectively. It takes 15 minutes for vLLM inferences. 

\paragraph{Hyperparameters.}

\begin{table}[h]
\centering
\begin{tabular}{@{}ll@{}}
\toprule
\textbf{Hyperparameter} & \textbf{Value} \\
\midrule
Optimizer & AdamW ($\beta_1{=}0.9$, $\beta_2{=}0.999$, $\epsilon{=}1\text{e-}8$) \\
Learning rate & 1e-7 \\
LR schedule & Cosine decay with 10\% warmup \\
Epoch & 1 \\
Per-device batch size & 32 \\
Gradient accumulation steps & 4 \\
Group size  & 8 \\
Max prompt length & 512 \\
Max completion length & 1024 \\
Generation temperature & 1.0 \\
Weight decay & 0.01 \\
Gradient clipping & 1.0 \\
Num iterations & 1 \\
Precision & bfloat16 \\
\bottomrule
\end{tabular}
\caption{Shared training hyperparameters across RL methods.}
\label{tab:shared_hyperparams}
\end{table}

\begin{table}[h]
\centering
\begin{tabular}{@{}lll@{}}
\toprule
\textbf{Hyperparameter} & \textbf{Symbol} & \textbf{Value} \\
\midrule
EMA momentum & $\tau$ & 0.999 \\
Similarity weight & $\alpha$ & 0.1 \\
Entropy coefficient & $\beta$ & 0.01 \\
Feature noise std & $\sigma$ & 0.02 \\
Reward threshold & N/A & 0.0\\
Predictor hidden dimension & N/A & 4096 \\
Predictor layers & N/A & 2 \\

\bottomrule
\end{tabular}
\caption{POPO hyperparameter configuration.}
\label{tab:POPO_hyperparams}
\end{table}

For GRPO, we use clipping parameter $\epsilon = 0.2$ and KL coefficient $\beta_{\text{KL}} = 0.01$.
Dr. GRPO applies the corrected advantage normalization.
DAPO uses clip-higher ratio $\epsilon_{\text{high}} = 0.28$, $\epsilon_{\text{low}} = 0.2$, with dynamic sampling and overlong filtering at 1.5$\times$ the mean completion length.

\begin{table}[!ht]
\centering
\begin{tabular}{@{}ll@{}}
\toprule
\textbf{Parameter} & \textbf{Value} \\
\midrule
Sampling temperature & 0.7 \\
Top-$p$ & 0.95 \\
Pass $k$ & 8 \\
Samples $n$ & 128 \\
Max output length (evaluation) & 4096 \\
Inference engine & vLLM \\
Mode & server \\
Tensor parallel size & 1 \\
\bottomrule
\end{tabular}
\caption{Evaluation configuration.}
\label{tab:eval_config}
\end{table}
\newpage

\section{Examples Responses from POPO and Other RL Algorithms}
\label{Examples responses from POPO and other RL algorithms}
\subsection{DeepSeek-R1-Distill-Qwen-7B Model on AIME 2025 Dataset }
\begin{questionbox}
[Input prompt: \\
$<|im \ start|>$ System. \\
Please reason step by step, and put your final answer within boxed $\brace{}$ $<|im \ end|>$.  \\
$<|im \ start |>$ User. \\
\textcolor{cyan}{"Question"} $<|im \ end|>$. \\
$<|im \ start|>$ Assistant ]
\textbf{Question}: 
An isosceles trapezoid has an inscribed circle tangent to each of its four sides. The radius of the circle is $3$, and the area of the trapezoid is $72$. Let the parallel sides of the trapezoid have lengths $r$ and $s$, with \(r \neq s\). Find $r^2+s^2$

\textbf{Ground-truth Answer}: 504
\end{questionbox}

\begin{responsebox}
\textbf{POPO Response}: 
First, recognize that an isosceles trapezoid with an inscribed circle is a tangential trapezoid, meaning the sum of the lengths of the two parallel sides equals the sum of the lengths of the two non-parallel sides.
Given the radius of the inscribed circle is 3, the height (distance) between the two parallel sides is 6.
The area of a trapezoid is given by the formula:
$ \text{Area} = \frac{(r + s)}{2} \times h $
Plugging in the known values:$72 = \frac{(r + s)}{2} \times 6$
Solving for $ r + s $:$r + s = 24$
Since the trapezoid is tangential, the sum of the lengths of the two parallel sides is equal to the sum of the lengths of the two non-parallel sides. Let the non-parallel sides each have length $ l $:$r + s = 2l$ $24 = 2l$ $l = 12$
To find $ r^2 + s^2 $, use the identity:$(r + s)^2 = r^2 + s^2 + 2rs$ $576 = r^2 + s^2 + 2rs$
Additionally, the Pythagorean theorem can be applied to the non-parallel sides, where the height is 6, and the difference between the bases is $ r - s $:
$ 12^2 = 6^2 + \left( \frac{r - s}{2} \right)^2 $
$ 144 = 36 + \left( \frac{r - s}{2} \right)^2 $
$ \left( \frac{r - s}{2} \right)^2 = 108 $
$ (r - s)^2 = 432 $
Now, solve the system of equations:$r^2 + s^2 + 2rs = 576$ $r^2 + s^2 - 2rs = 432$  
Adding both equations:
$ 2(r^2 + s^2) = 1008 $
$ r^2 + s^2 = \boxed{504} \quad \textcolor{green}{Correct} $
\end{responsebox}

\begin{responsebox}
\textbf{GRPO Response}:
To solve the problem, we'll use the properties of an isosceles trapezoid with an inscribed circle.
First, recall that for a trapezoid to have an inscribed circle (be tangential), the sum of the lengths of the two parallel sides must equal the sum of the lengths of the two non-parallel sides.
Let’s denote: The lengths of the two parallel sides as $ r $ and $ s $ (with $ r \neq s $).
Given that the radius of the inscribed circle is 3, the diameter is 6. However, in an isosceles tangential trapezoid, the diameter isn't directly the height. Instead, the radius is related to the area.The area $ A $ of a trapezoid can be expressed as: $A = \frac{(r + s)}{2} \times h$ where $ h $ is the height.
Wait, actually, no. In a tangential quadrilateral, the area is equal to the product of the semiperimeter and the radius: $A = (r + s + 2a)/2 \times r$ but in this case, for a trapezoid, the semiperimeter $ p $ is:$p = \frac{r + s + 2a}{2}$ and the area is also given by:$A = p \times r = \frac{r + s + 2a}{2} \times r$
However, we also know that in an isosceles tangential trapezoid, the sum of the two parallel sides equals the sum of the two non-parallel sides:$r + s = 2a$.Therefore, substituting into the area formula: $A = \frac{(r + s) + (r + s)}{2} \times r = (r + s) \times r$
But we are given that the area $ A = 72 $, so:
$ (r + s) \times r = 72 $
Let's consider the height $ h $ of the trapezoid.
In a tangential trapezoid, the height is equal to twice the radius:
$ h = 2r = 6 $
Therefore, the area can also be expressed as:
$ A = \frac{r + s}{2} \times h = \frac{r + s}{2} \times 6 = 3(r + s) $
Given that $ A = 72 $, we have:
$ 3(r + s) = 72 $
$ r + s = 24 $
Now, we have two equations:
 $ r + s = 24 $
$ (r + s) \times r = 72 $
From the first equation, $ r + s = 24 $. Plugging this into the second equation:
$ 24 \times r = 72 $
$ r = 3 $
Substituting $ r = 3 $ back into $ r + s = 24 $:
$ s = 24 - 3 = 21 $
Finally, we need to find $ r^2 + s^2 $:
$ r^2 + s^2 = 3^2 + 21^2 = 9 + 441 = 450 $
The value of $ r^2 + s^2 $ is \boxed{450}.\quad \textcolor{red}{Wrong}
\end{responsebox}

\begin{responsebox}
\textbf{Dr. GRPO Response}:
To solve the problem, I need to use the properties of an isosceles trapezoid with an inscribed circle (incircle). Here are the key steps:
1. \textbf{Understanding the Properties:}
   - An isosceles trapezoid with an incircle must be tangential, meaning the sum of the lengths of the two opposite sides is equal. In this case, the sum of the two parallel sides (bases) equals the sum of the two non-parallel sides (legs).
2. \textbf{Expressing the Perimeter:}
   - Let the lengths of the two parallel sides be \( r \) and \( s \), and the lengths of the legs be \( l \). Since the trapezoid is tangential, \( r + s = 2l \).
3. \textbf{Calculating the Area:}
   - The area of a trapezoid is given by \( \text{Area} = \frac{(r + s)}{2} \times h \), where \( h \) is the height. Given the area is 72, we have \( \frac{(r + s)}{2} \times h = 72 \).
   - Since the trapezoid is tangential and has an incircle with radius 3, the height \( h = 2 \times \text{radius} = 6 \).
4. \textbf{Solving for the Sum of the Bases:}
   - Substituting \( h = 6 \) into the area formula: \( \frac{(r + s)}{2} \times 6 = 72 \) leads to \( r + s = 24 \).
5. \textbf{Expressing the Legs in Terms of the Bases:}
   - From the tangential property, \( r + s = 2l \) implies \( l = 12 \).
6. \textbf{Using the Pythagorean Theorem:}
   - The legs of the trapezoid can be expressed in terms of the height and the difference of the bases. Let \( d = \frac{r - s}{2} \). Then, \( l^2 = h^2 + d^2 \).
   - Substituting \( l = 12 \) and \( h = 6 \), we find \( d = \sqrt{12^2 - 6^2} = \sqrt{108} = 6\sqrt{3} \).
7. \textbf{Calculating \( r^2 + s^2 \):}
   - Using the identities \( (r + s)^2 = r^2 + 2rs + s^2 \) and \( (r - s)^2 = r^2 - 2rs + s^2 \), we can solve for \( r^2 + s^2 \).
   - Adding these equations gives \( 2(r^2 + s^2) = (r + s)^2 + (r - s)^2 \).
   - Substituting \( r + s = 24 \) and \( r - s = 12\sqrt{3} \), we get \( r^2 + s^2 = \frac{24^2 + (12\sqrt{3})^2}{2} = \frac{576 + 432}{2} = 504 \).
\[
\boxed{504}\quad \textcolor{green}{Correct}
\]
\end{responsebox}

\begin{responsebox}
\textbf{DAPO Response}:
To solve the problem, let's start by recalling some properties of an isosceles trapezoid with an inscribed circle (a tangential trapezoid).
First, in a tangential quadrilateral (a quadrilateral with an inscribed circle), the sum of the lengths of the two opposite sides is equal. For an isosceles trapezoid, this means the sum of the two parallel sides (the bases) is equal to the sum of the two non-parallel sides (the legs). Let's denote the lengths of the two bases as $ r $ and $ s $, and the lengths of the legs as $ l $. So, we have:
$ r + s = 2l $
Next, we know the radius of the inscribed circle is $ 3 $, so the area $ A $ of the trapezoid can be expressed as:
$ A = (r + s) \times r = (r + s) \times 3 $
However, the area is also given as $ 72 $, so:
$ 72 = (r + s) \times 3 \implies r + s = 24 $
From the earlier equation $ r + s = 2l $, we find that:
$ 2l = 24 \implies l = 12 $
Now, the height $ h $ of the trapezoid is twice the radius, so:
$ h = 2 \times 3 = 6 $
Using the Pythagorean theorem, we can relate the difference of the bases to the legs and the height. The difference between the bases is $ |r - s| $, and half of that is:
$ \frac{|r - s|}{2} $
$ \left( \frac{|r - s|}{2} \right)^2 + h^2 = l^2 $
Substituting the known values:
$ \left( \frac{|r - s|}{2} \right)^2 + 6^2 = 12^2 $
Simplifying:
$ \left( \frac{|r - s|}{2} \right)^2 = 144 - 36 = 108 $
$ \frac{|r - s|}{2} = \sqrt{108} = 6\sqrt{3} $
$ |r - s| = 12\sqrt{3} $
Now, we have a system of equations:
$ r + s = 24 $
$ |r - s| = 12\sqrt{3} $
Assuming $ r > s $, we can write:
$ r - s = 12\sqrt{3} $
Now, solve for $ r $ and $ s $:
Adding the equations:
$ 2r = 24 + 12\sqrt{3} \implies r = 12 + 6\sqrt{3} $
Subtracting the equations:
$ 2s = 24 - 12\sqrt{3} \implies s = 12 - 6\sqrt{3} $
Finally, calculate $ r^2 + s^2 $:
$ r^2 + s^2 = (12 + 6\sqrt{3})^2 + (12 - 6\sqrt{3})^2 $
Expanding both squares:
$ (144 + 144\sqrt{3} + 108) + (144 - 144\sqrt{3} + 108) = 144 + 108 + 144 + 108 = \boxed{504}.\quad \textcolor{green}{Correct} $
\end{responsebox}

\begin{responsebox}
\textbf{BASE Response}:
First, since the trapezoid is isosceles and has an inscribed circle, it must satisfy the condition that the sum of the lengths of the two parallel sides (bases) is equal to the sum of the lengths of the two non-parallel sides (legs).
Let the lengths of the parallel sides be \( r \) and \( s \), and the lengths of the legs be \( l \). Therefore, \( r + s = 2l \).
Given that the radius of the inscribed circle is \( 3 \), the height \( h \) of the trapezoid is \( 2 \times 3 = 6 \).
The area of the trapezoid is given by \( \frac{(r + s)}{2} \times h = 72 \). Substituting \( h = 6 \), we have:
\[
\frac{(r + s)}{2} \times 6 = 72 \implies r + s = 24
\]
Since the trapezoid is isosceles and has an inscribed circle, the legs are equal in length, and we have:
\[
l = \frac{r + s}{2} = 12
\]
Using the Pythagorean theorem with one of the legs and the difference of the bases:
\[
l^2 = \left( \frac{r - s}{2} \right)^2 + h^2 \implies 12^2 = \left( \frac{r - s}{2} \right)^2 + 6^2 \implies 144 = \left( \frac{r - s}{2} \right)^2 + 36
\]
\[
\left( \frac{r - s}{2} \right)^2 = 108 \implies \left( r - s \right)^2 = 432
\]
Finally, calculating \( r^2 + s^2 \):
\[
(r + s)^2 = r^2 + 2rs + s^2 = 576 \implies r^2 + s^2 = 576 - 2 \times 72 = \boxed{432}\quad \textcolor{red}{Wrong}
\]
\end{responsebox}

\begin{responsebox}
\textbf{SFT Response}:
Given an isosceles trapezoid with an inscribed circle, the radius of the circle is $3$, and the area of the trapezoid is $72$. We need to find $r^2+s^2$, where $r$ and $s$ are the lengths of the two parallel sides, with $r \ne s$. First, since the circle is tangent to both parallel sides of the trapezoid, the distance between those two sides, which is the height of the trapezoid, must be twice the radius. Therefore, the height is $h=2\cdot 3=6$. Next, we use the area formula for a trapezoid: $\text{Area}=\frac12(r+s)h$. Substituting the given area and the height, we get $72=\frac12(r+s)\cdot 6$, which simplifies to $72=3(r+s)$, so $r+s=24$. At this point, we have one useful equation. Now we use the fact that the trapezoid has an inscribed circle. For a tangential trapezoid, the sum of the two bases equals the length of each leg, so we write $r+s=l$. Since we already found that $r+s=24$, it follows that the leg length is $l=24$. This gives us another quantity to work with. Because the trapezoid is isosceles, if we drop perpendiculars from the endpoints of the shorter base to the longer base, we form two congruent right triangles on the sides. If we let $s$ be the longer base and $r$ be the shorter base, then the horizontal leg of each right triangle is $\frac{s-r}{2}$, the vertical leg is the height $6$, and the hypotenuse is the leg of the trapezoid, which has length $24$. By the Pythagorean theorem, we have $l=\sqrt{\left(\frac{s-r}{2}\right)^2+6^2}$. Substituting $l=24$, we obtain $24=\sqrt{\left(\frac{s-r}{2}\right)^2+36}$. Squaring both sides gives $576=\left(\frac{s-r}{2}\right)^2+36$. Subtracting $36$ from both sides, we get $\left(\frac{s-r}{2}\right)^2=540$. Taking square roots, we find $\frac{s-r}{2}=\sqrt{540}$. Since $540=36\cdot 15$, this becomes $\frac{s-r}{2}=6\sqrt{15}$, and therefore $s-r=12\sqrt{15}$. Now we have the system $r+s=24$ and $s-r=12\sqrt{15}$. Adding the two equations gives $(r+s)+(s-r)=24+12\sqrt{15}$, so $2s=24+12\sqrt{15}$, which means $s=12+6\sqrt{15}$. Subtracting the second equation from the first gives $(r+s)-(s-r)=24-12\sqrt{15}$, so $2r=24-12\sqrt{15}$, and thus $r=12-6\sqrt{15}$. Finally, we compute $r^2+s^2$. Instead of squaring each expression separately, we use the identity $(r+s)^2+(s-r)^2=2(r^2+s^2)$. Therefore, $r^2+s^2=\frac{(r+s)^2+(s-r)^2}{2}$. Substituting the values we found, we get $r^2+s^2=\frac{24^2+(12\sqrt{15})^2}{2}=\frac{576+2160}{2}=\frac{2736}{2}=1368$. Hence, the value of $r^2+s^2$ is $\boxed{1368}$.\textcolor{red}{Wrong}
\end{responsebox}

\vspace{6pt}
\subsection{Qwen2.5-Math-7B Model on Olympiad Dataset}
\begin{questionbox} [Input prompt: \\
$<|im \ start|>$ System. \\
Please reason step by step, and put your final answer within boxed $\brace{}$ $<|im \ end|>$.  \\
$<|im \ start |>$ User. \\
\textcolor{cyan}{"Question"} $<|im \ end|>$. \\
$<|im \ start|>$ Assistant ]
\textbf{Question}: 
For a positive integer $a$, define a sequence of integers \(\ x_{1}, x_{2}, \ldots\) by letting $x_{1}=a$ and $x_{n+1}=2 x_{n}+1$ for \(\ n \geq 1\). Let $y_{n}=2^{x_{n}}-1$. Determine the largest possible $k$ such that, for some positive integer $a$, the numbers \(\ y_{1}, \ldots, y_{k}\) are all prime.

\textbf{Ground-truth Answer}: 2
\end{questionbox}

\begin{responsebox}
\textbf{POPO Response}:
To determine the largest possible \( k \) such that, for some positive integer \( a \), the numbers \( y_1, y_2, \ldots, y_k \) are all prime, we start by analyzing the given sequence and the properties of the numbers involved.
Given:
- \( x_1 = a \) (where \( a \) is a positive integer)
- \( x_{n+1} = 2x_n + 1 \) for \( n \geq 1 \)
- \( y_n = 2^{x_n} - 1 \)
We need to find the largest \( k \) such that all \( y_1, y_2, \ldots, y_k \) are prime numbers.
First, let's compute the first few terms of the sequence \( \{x_n\} \) and corresponding \( y_n \):
1. \textbf{For \( n = 1 \):}
   - \( x_1 = a \)
   - \( y_1 = 2^{a} - 1 \)
   - \( y_1 \) must be prime. This implies that \( a \) itself must be prime since \( 2^a - 1 \) is prime only if \( a \) is prime (otherwise, \( 2^a - 1 \) would be composite).
2. \textbf{For \( n = 2 \):}
   - \( x_2 = 2x_1 + 1 = 2a + 1 \)
   - \( y_2 = 2^{2a+1} - 1 \)
   - \( y_2 \) must also be prime.
3. \textbf{For \( n = 3 \):}
   - \( x_3 = 2x_2 + 1 = 2(2a + 1) + 1 = 4a + 3 \)
   - \( y_3 = 2^{4a+3} - 1 \)
   - \( y_3 \) must be prime.
4. \textbf{For \( n = 4 \):}
   - \( x_4 = 2x_3 + 1 = 2(4a + 3) + 1 = 8a + 7 \)
   - \( y_4 = 2^{8a+7} - 1 \)
   - \( y_4 \) must be prime.
Continuing this pattern, each subsequent \( y_n \) becomes exponentially larger. The challenge is to find an \( a \) such that all \( y_1 \) through \( y_k \) are prime.
To maximize \( k \), we should start with the smallest possible prime \( a \), which is \( a = 2 \):
- \( y_1 = 2^2 - 1 = 3 \) (prime)
- \( y_2 = 2^{5} - 1 = 31 \) (prime)
- \( y_3 = 2^{11} - 1 = 2047 \) (not prime, since \( 2047 = 23 \times 89 \))
Thus, with \( a = 2 \), we only get \( k = 2 \).
Trying \( a = 3 \):

- \( y_1 = 2^3 - 1 = 7 \) (prime)
- \( y_2 = 2^{7} - 1 = 127 \) (prime)
- \( y_3 = 2^{15} - 1 = 32767 \) (not prime, since \( 32767 = 7 \times 31 \times 151 \))
Again, \( k = 2 \).
Trying \( a = 5 \):
- \( y_1 = 2^5 - 1 = 31 \) (prime)
- \( y_2 = 2^{11} - 1 = 2047 \) (not prime)
Thus, \( k = 1 \).
It appears that regardless of the choice of \( a \), achieving \( k \geq 3 \) is not possible because \( y_3 \) tends to be composite. Therefore, the maximum value of \( k \) for which all \( y_1 \) through \( y_k \) are prime is \( k = 2 \).
The largest possible \( k \) such that \( y_1, \ldots, y_k \) are all prime is \( \boxed{2} \).\quad \textcolor{green}{Correct}
\end{responsebox}

\begin{responsebox}
\textbf{GRPO Response}:
To solve this problem, we need to analyze the sequence \( x_n \) and the corresponding sequence \( y_n \).
1. First, let's define the sequence \( x_n \):
   - \( x_1 = a \)
   - \( x_{n+1} = 2x_n + 1 \)
2. From this, we can derive the general form of \( x_n \):
   - \( x_2 = 2x_1 + 1 = 2a + 1 \)
   - \( x_3 = 2x_2 + 1 = 2(2a + 1) + 1 = 4a + 3 \)
   - \( x_4 = 2x_3 + 1 = 2(4a + 3) + 1 = 8a + 7 \)
   - By induction, we can see that \( x_n = 2^{n-1}a + 2^{n-1} - 1 = 2^{n-1}(a + 1) - 1 \).
3. Next, we define \( y_n \):
   - \( y_n = 2^{x_n} - 1 \)
   - Substituting \( x_n \) into \( y_n \), we get \( y_n = 2^{2^{n-1}(a + 1) - 1} - 1 \).
4. To find the largest possible \( k \) such that \( y_1, y_2, \ldots, y_k \) are all prime, we need to check the primality of the first few terms of \( y_n \).
5. Let's start with \( y_1 \):
   - \( y_1 = 2^a - 1 \)
   - For \( y_1 \) to be prime, \( a \) must be such that \( 2^a - 1 \) is prime. The smallest such \( a \) is 2, so \( y_1 = 2^2 - 1 = 3 \), which is prime.
6. Next, let's check \( y_2 \):
   - \( y_2 = 2^{2a + 1} - 1 \)
   - For \( y_2 \) to be prime, \( 2a + 1 \) must be a prime exponent in the Mersenne prime form. If \( a = 2 \), then \( y_2 = 2^{2 \cdot 2 + 1} - 1 = 2^5 - 1 = 31 \), which is prime.
7. Next, let's check \( y_3 \):
   - \( y_3 = 2^{4a + 3} - 1 \)
   - For \( y_3 \) to be prime, \( 4a + 3 \) must be a prime exponent in the Mersenne prime form. If \( a = 2 \), then \( y_3 = 2^{4 \cdot 2 + 3} - 1 = 2^{11} - 1 = 2047 = 23 \cdot 89 \), which is not prime.
So, the largest possible \( k \) such that \( y_1, y_2, \ldots, y_k \) are all prime is 2.
Therefore, the answer is \( \boxed{2} \).\quad \textcolor{green}{Correct}
\end{responsebox}

\subsection{Qwen-Math-1.5B Model on the MATH-500 Dataset}
\begin{questionbox} [Input prompt: \\
$<|im \ start|>$ System. \\
Please reason step by step, and put your final answer within boxed $\brace{}$ $<|im \ end|>$.  \\
$<|im \ start |>$ User. \\
\textcolor{cyan}{"Question"} $<|im \ end|>$. \\
$<|im \ start|>$ Assistant ]
\textbf{Question}: 
How many positive whole-number divisors does 196 have?

\textbf{Ground-truth Answer}: 9
\end{questionbox}

\begin{responsebox}
\textbf{POPO Response}:
To find how many positive whole-number divisors $ 196 $ has, we first factor it into primes.
We have:
$ 196 = 14^2 = (2 \times 7)^2 = 2^2 \times 7^2 $
Now, use the divisor-counting formula:
If a number has prime factorization $ p^a q^b $, then the number of positive divisors is:
$ (a+1)(b+1) $
Here, $ a = 2 $ and $ b = 2 $, so the number of positive divisors is:
$ (2+1)(2+1) = 3 \times 3 = 9 $
Therefore, the number of positive whole-number divisors of $ 196 $ is $\boxed{9}.\quad \textcolor{green}{Correct}$
\end{responsebox}

\subsection{Deepseek-Math-7B Model on AMC 23 dataset}
\begin{questionbox} [ Input prompt: \\ 
$<|im \ start|>$ System. \\ 
You are a helpful assistant. Please solve the following math problem step by step and put your final answer within $<|im \ end|>$. \\ 
$<|im \ start|>$ User. \\ 
\textcolor{cyan}{"Question"} $<|im \ end|>$. \\ 
$<|im \ start|>$ Assistant.
]
\textbf{Question}: 
How many complex numbers satisfy the equation $z^5=\overline{z}$, where $\overline{z}$ is the conjugate of the complex number $z$?

\textbf{Ground-truth Answer:} 7
\end{questionbox}

\begin{responsebox}
\textbf{POPO Response}:
We want to find the number of complex numbers $ z $ satisfying
$ z^5=\overline{z} $,
where $ \overline{z} $ denotes the complex conjugate of $ z $.
To analyze this equation, write the complex number $ z $ in polar form:
$ z=re^{i\theta} $,
where $ r \ge 0 $ and $ \theta $ is a real angle.
Then:
$ z^5 = r^5 e^{i5\theta} $
and
$ \overline{z} = r e^{-i\theta} $
So the equation $ z^5=\overline{z} $ becomes:
$ r^5 e^{i5\theta} = r e^{-i\theta} $
Now we consider two cases.
First, if $ r=0 $, then $ z=0 $, and indeed:
$ z^5=0=\overline{z} $
So $ z=0 $ is one solution.
Now suppose $ r>0 $. Then we may divide both sides of
$ r^5 e^{i5\theta} = r e^{-i\theta} $
by $ r e^{-i\theta} $, giving:
$ r^4 e^{i6\theta} = 1 $
For a complex number to equal $ 1 $, its magnitude must be $ 1 $ and its argument must be a multiple of $ 2\pi $. Therefore, we must have:
$ r^4=1 $
and
$ e^{i6\theta}=1 $
Since $ r \ge 0 $, the equation $ r^4=1 $ implies:
$ r=1 $
And from
$ e^{i6\theta}=1 $,
we get:
$ 6\theta=2\pi k $
for some integer $ k $.
Thus:
$ \theta=\frac{2\pi k}{6}=\frac{\pi k}{3} $
As $ k=0,1,2,3,4,5 $, this gives the $ 6 $ distinct solutions on the unit circle:
$ e^{0i}, e^{i\pi/3}, e^{i2\pi/3}, e^{i\pi}, e^{i4\pi/3}, e^{i5\pi/3} $
So there are:
$ 6 $
nonzero solutions, together with the zero solution.
Hence, the total number of complex numbers satisfying the equation is:
$ 6+1=7 $.
Therefore, the answer is $\boxed{7}.\quad \textcolor{green}{Correct}$
\end{responsebox}

\vspace{6pt}

\end{document}